\documentclass{new_tlp}
\usepackage{mathptmx}

\usepackage{amssymb}
\usepackage{amsmath}
\usepackage[ruled,linesnumbered,vlined]{algorithm2e}

\usepackage{booktabs}

\usepackage{hyperref}
\usepackage{tikz}
\usetikzlibrary{positioning, fit, arrows, backgrounds}

\newtheorem{example}{Example}
\newtheorem{theorem}{Theorem}

\newcommand{\tuple}[1]{\ensuremath{\langle{#1}\rangle}\xspace}

\usepackage{listings}

\usepackage{multirow}
\usepackage{rotating}

\makeatletter
\lst@Key{countblanklines}{true}[t]{\lstKV@SetIf{#1}\lst@ifcountblanklines}
\lst@AddToHook{OnEmptyLine}{%
    \lst@ifnumberblanklines\else%
    \lst@ifcountblanklines\else%
    \advance\c@lstnumber-\@ne\relax%
    \fi%
    \fi}
\makeatother
\lstdefinelanguage{asp}{
    breakatwhitespace=true,
    captionpos=b,
    numbers=left,
    numbersep=5pt,
    numberblanklines=false,
    countblanklines=false,
    commentstyle=\colour{gray},
    frame=bt, framexbottommargin=5pt, framextopmargin=5pt,
    aboveskip=5pt, belowskip=5pt,
    abovecaptionskip=10pt,
    basicstyle=\small\ttfamily
}
\lstnewenvironment{asp}{
    \lstset{
        language=asp,
        showstringspaces=false,
        keepspaces=true,
        formfeed=\newpage,
        tabsize=4,
        numbers=none,
        breaklines=true,
        literate={~} {$\sim$}{1},
        frame=none,
    }
}{}

\newcommand\bcmdtab{\noindent\bgroup\tabcolsep=0pt%
  \begin{tabular}{@{}p{10pc}@{}p{20pc}@{}}}
\newcommand\ecmdtab{\end{tabular}\egroup}

\title[Logic-Guided Data Extraction with ASP and LLMs]{Logic-Guided Data Extraction with Answer Set Programming and Large Language Models}

\author[Alviano et al.]
{
Mario Alviano\textsuperscript{1}, Lorenzo Grillo\textsuperscript{1,2}, Nicola Leone\textsuperscript{1} and Fabrizio {Lo Scudo}\textsuperscript{1} \\
(1) University of Calabria, Rende, Italy\\
(2) University of Campania Luigi Vanvitelli, Caserta, Italy\\
\email{\{mario.alviano,fabrizio.loscudo\}@unical.it}
}

\begin{document}

\label{firstpage}

\maketitle

\begin{abstract}
When Large Language Models (LLMs) are used for semantic data extraction from unstructured text, producing candidate relational facts from natural language, they may remain unreliable for tasks requiring complex combinatorial reasoning and global consistency. This paper proposes a logic-guided data extraction framework combining LLM-based extraction with Answer Set Programming (ASP). The LLM produces candidate facts, whereas ASP performs validation, inference, consistency checking, and control. Unlike existing pipelines that query the LLM independently for all target predicates, the proposed approach uses ASP reasoning to identify which predicates are logically admissible at each stage and to guide extraction queries. By interleaving LLM calls with ASP derivation, the framework infers logically implied facts without further extraction and detects inconsistencies early. We formalize the pipeline and prove that, under mild assumptions, it is equivalent to the baseline approach with respect to the final extracted facts, while requiring fewer LLM calls. We also introduce a caching mechanism for logic-based control queries, exploiting monotonicity of conjunctive queries over incrementally constructed fact sets to reduce solver invocations. Experiments on ASP-derived benchmarks show that the framework reduces LLM calls and improves extraction quality by mitigating spurious outputs, demonstrating the value of non-monotonic logic programming for controlled semantic extraction.
\end{abstract}


\section{Introduction}

Large Language Models (LLMs), surveyed by \citeNP{DBLP:journals/tist/NaveedKQSAUABM25}, have recently been adopted for semantic data extraction and parsing from unstructured text in the works by \citeNP{DBLP:journals/corr/abs-2312-17617}, \citeNP{10.1145/3748304}, \citeNP{DBLP:journals/corr/abs-2403-02901}, and \citeNP{DBLP:journals/tkde/ZhangLDBL23}, enabling the automatic construction of structured representations such as relational facts or knowledge bases.
In logic-based extraction pipelines, 
such as the [LLM]+ASP framework proposed by \citeNP{DBLP:conf/acl/YangI023} and the LLMASP framework by \citeNP{DBLP_conf_ijcai_AlvianoGSR25}, 
extraction is commonly performed independently for each target predicate, with logical validation and inference applied only afterward.
While effective, this design often requires issuing a large number of extraction queries and coping with incomplete or spurious outputs, especially in domains where predicates are logically interdependent.
Crucially, this approach decouples extraction from \textit{reasoning}: the LLM is used merely as a text parser, ignoring the deductive steps that could actively guide the information retrieval.

This work lies at the intersection of Neuro-Symbolic AI and Information Extraction.
Unlike probabilistic neuro-symbolic frameworks such as DeepProbLog by \citeNP{DBLP:conf/nips/ManhaeveDKDR18} and NeurASP by \citeNP{DBLP:conf/ijcai/YangIL20} that integrate neural outputs via probabilistic semantics, we focus on utilizing pre-trained LLMs as black-box oracles.
In this modular setting, existing approaches typically employ logic only for post-hoc validation, as proposed by \citeNP{DBLP:conf/ijcai/EiterGHO23}, or for iterative refinement, as explored by \citeNP{DBLP:conf/nips/MadaanTGHGW0DPY23} and \citeNP{DBLP:conf/nips/ShinnCGNY23}.
Answer Set Programming (ASP), introduced by \citeNP{MarekT99} and \citeNP{Niemela99}, is particularly well-suited for this task due to its ability to represent complex relational structures and non-monotonic dependencies, as demonstrated by \citeNP{DBLP:journals/debu/CaliCGL02} and \citeNP{DBLP:conf/lpnmr/LeoneEFFGGGKILLLNRRST05}.
However, when ASP is used solely as a reasoning component after extraction (as in LLMASP), the logical structure of the domain is not exploited to guide the extraction process itself.
As a consequence, logical dependencies that could inform extraction decisions are ignored during the extraction phase.

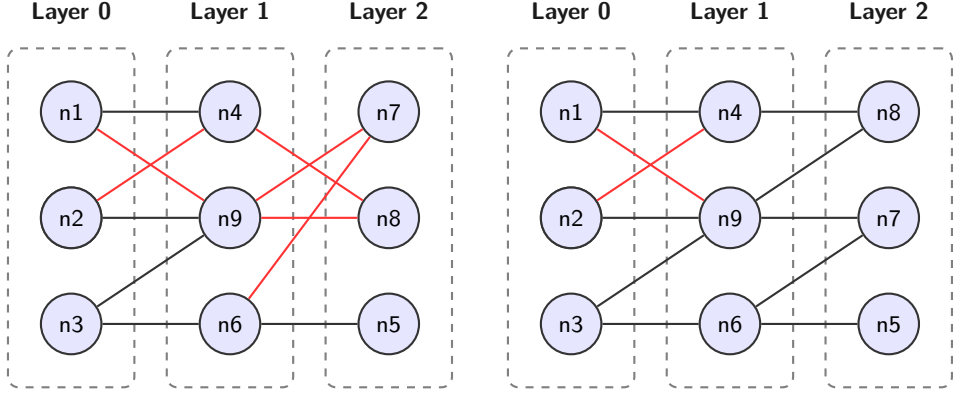
\begin{figure}[t]
    \centering

\begin{tikzpicture}[
    scale=0.7,
    >=stealth,
    dot/.style={circle, draw=black!80, fill=blue!10, thick, minimum size=8mm, inner sep=0pt, font=\sffamily\small},
    layerbox/.style={draw=black!50, dashed, thick, rounded corners, inner sep=12pt},
    labeltxt/.style={font=\sffamily\bfseries, text=black!90}
]

\node[dot](a) at (0,0) {ciao};
\node[dot] (n1) at (0, 2) {n1};
\node[dot] (n2) at (0, 0) {n2};
\node[dot] (n3) at (0, -2) {n3};

\node[layerbox, fit=(n1)(n2)(n3)] (l0) {};
\node[labeltxt, above=2mm of l0] {Layer 0};

\node[dot] (n4) at (3, 2) {n4};
\node[dot] (n5) at (3, 0) {n9};  
\node[dot] (n6) at (3, -2) {n6};

\node[layerbox, fit=(n4)(n5)(n6)] (l1) {};
\node[labeltxt, above=2mm of l1] {Layer 1};

\node[dot] (n7) at (6, 2) {n7};
\node[dot] (n8) at (6, 0) {n8};
\node[dot] (n9) at (6, -2) {n5};

\node[layerbox, fit=(n7)(n8)(n9)] (l2) {};
\node[labeltxt, above=2mm of l2] {Layer 2};


\draw[thick, red!80] (n1) to (n5);
\draw[thick, black!80] (n1) to (n4);

\draw[thick, black!80] (n2) to (n5);

\draw[thick, black!80] (n3) -- (n5);
\draw[thick, black!80] (n3) -- (n6);

\draw[thick, red!80] (n4) -- (n2);
\draw[thick, red!80] (n4) -- (n8);

\draw[thick, red!80] (n5) -- (n7);
\draw[thick, red!80] (n5) -- (n8);

\draw[thick, red!80] (n6) -- (n7);
\draw[thick, black!80] (n6) -- (n9);

\end{tikzpicture}
\qquad
\begin{tikzpicture}[
    scale=0.7,
    >=stealth,
    dot/.style={circle, draw=black!80, fill=blue!10, thick, minimum size=8mm, inner sep=0pt, font=\sffamily\small},
    layerbox/.style={draw=black!50, dashed, thick, rounded corners, inner sep=12pt},
    labeltxt/.style={font=\sffamily\bfseries, text=black!90}
]

\node[dot](a) at (0,0) {ciao};
\node[dot] (n1) at (0, 2) {n1};
\node[dot] (n2) at (0, 0) {n2};
\node[dot] (n3) at (0, -2) {n3};

\node[layerbox, fit=(n1)(n2)(n3)] (l0) {};
\node[labeltxt, above=2mm of l0] {Layer 0};

\node[dot] (n4) at (3, 2) {n4};
\node[dot] (n5) at (3, 0) {n9};
\node[dot] (n6) at (3, -2) {n6};

\node[layerbox, fit=(n4)(n5)(n6)] (l1) {};
\node[labeltxt, above=2mm of l1] {Layer 1};

\node[dot] (n7) at (6, 0) {n7};
\node[dot] (n8) at (6, 2) {n8};
\node[dot] (n9) at (6, -2) {n5};

\node[layerbox, fit=(n7)(n8)(n9)] (l2) {};
\node[labeltxt, above=2mm of l2] {Layer 2};


\draw[thick, red!80] (n1) to (n5);
\draw[thick, black!80] (n1) to (n4);

\draw[thick, black!80] (n2) to (n5);

\draw[thick, black!80] (n3) -- (n5);
\draw[thick, black!80] (n3) -- (n6);

\draw[thick, red!80] (n4) -- (n2);
\draw[thick, black!80] (n4) -- (n8);

\draw[thick, black!80] (n5) -- (n7);
\draw[thick, black!80] (n5) -- (n8);

\draw[thick, black!80] (n6) -- (n7);
\draw[thick, black!80] (n6) -- (n9);

\end{tikzpicture}
\vspace{1em}

    \caption{{Comparison of a three-layer graph before (left) and after (right) applying a crossing minimization strategy. Intersecting edges are highlighted in red. Reordering the nodes, specifically in Layer 2, significantly reduces the total number of crossings and improves the readability of the layout.}}
    \label{fig:crossing}
\end{figure}

\begin{example}[LLM Extraction with Logic-Based Reasoning]\label{ex:motivation}
Consider the \emph{Crossing Minimization} problem, where one seeks an ordering of nodes within layers of a graph that minimizes edge crossings {(see Figure~\ref{fig:crossing})}.
Solving such combinatorial reasoning tasks directly with a LLM is challenging, as it requires enforcing global consistency constraints and exploring complex structural dependencies.
Rather than expecting the LLM to solve reasoning tasks implicitly, better results can be obtained through a clear division of the work:
the LLM extracts candidate facts from text, and ASP performs reasoning and consistency checks.
Here, the extraction task consists in building a structured graph representation from natural language descriptions.
For example, a text may describe a (possibly layered) undirected graph by mentioning nodes, edges, and, when relevant, information about layers.
LLM-based extractors can be queried to produce facts such as
\lstinline|edge(n1,n9)|, \lstinline|layer_size(0,3)|, or \lstinline|in_layer(0,n1)|.
However, issuing extraction requests for all target predicates independently, as in LLMASP, can be both inefficient and error-prone:
it may lead to redundant oracle calls and to spurious extractions, such as layer assignments when no layering information is present in the text.
This motivates logic-based admissibility conditions that determine when extraction queries should be issued.
We will use this layered-graph scenario as a running example throughout the paper (see Example~\ref{ex:running:1}).
\hfill$\blacksquare$
\end{example}

In this paper, we propose a logic-guided data extraction framework in which ASP plays an active role in controlling and structuring the interaction with an LLM-based extractor.
Rather than invoking the LLM independently for all target predicates (a bottom-up approach), our framework uses ASP reasoning to determine which predicates are logically admissible at each stage of the extraction process.
Parallel to the reasoning-acting decomposition introduced in ReAct by \citeNP{DBLP:conf/iclr/YaoZYDSN023}, and conceptually akin to the evaluation of external atoms in HEX-programs described by \citeNP{DBLP:journals/ki/EiterGIKRSW18} and to the execution of external actions in ActHEX introduced by \citeNP{DBLP:conf/lpnmr/FinkGIRS13}, the extraction is interleaved with ASP reasoning and derivation, allowing logically implied facts to be inferred without any additional extraction and enabling early consistency checks on partial data.
This query-driven behavior is similar in spirit to goal-directed ASP evaluation as developed by \citeNP{DBLP:conf/datalog/AriasCCG19}, ensuring that extraction resources are spent only on facts that are consistent with the current partial model.
We formalize the proposed pipeline and show that, under mild assumptions on the extraction oracle, it is equivalent to a baseline LLM-based extraction approach in terms of the final extracted facts, while never increasing the number of calls to the LLM and potentially reducing them.
We further introduce a caching mechanism for logic-based control queries, exploiting monotonicity properties of conjunctive queries over incrementally constructed fact sets to reduce the number of ASP solver invocations.
An experimental evaluation on benchmarks derived from standard Answer Set Programming domains demonstrates that the proposed framework substantially reduces the number of LLM calls and, in practice, improves extraction quality by mitigating spurious outputs. 
These results show that non-monotonic logic programming can effectively serve as a control mechanism for semantic data extraction from text, complementing the generative capabilities of LLMs.

\paragraph{Contributions.}
This paper makes the following contributions:
(i) we introduce a clean formal abstraction of LLM-based extraction pipelines, exemplified by LLMASP;
(ii) we propose a logic-guided extraction framework that interleaves extraction and ASP reasoning via guard-based admissibility conditions;
(iii) we establish equivalence with the baseline pipeline under a guard-respecting oracle assumption and prove efficiency properties;
(iv) we present an execution algorithm with sound guard caching;
(v) we experimentally evaluate the framework on ASP-derived benchmarks.

\paragraph{Structure.}
The remainder of the paper is organized as follows. 
Section~\ref{sec:background} introduces the baseline extraction setting and notation. 
Section~\ref{sec:guided} presents the logic-guided extraction framework and its formal properties. 
Section~\ref{sec:implementation} presents an execution algorithm for logic-guided extraction and its realization of the proposed semantics.
Section~\ref{sec:experiments} reports the experimental evaluation, and
Section~\ref{sec:conclusion} concludes.

\section{Preliminaries and Baseline Setting}\label{sec:background}

This section introduces a formal abstraction of LLM-based extraction pipelines,
including existing approaches such as LLMASP.
While inspired by prior work, this abstraction is new and will serve as the
foundation for the logic-guided framework introduced in Section~3.
The presentation is intentionally minimal and focuses only on concepts required to formalize logic-guided extraction.

\paragraph{Text, Predicates, and Facts.}
Let $T$ denote a finite unstructured text document, and $\mathcal{T}$ be the domain of text documents.  
Let $\mathcal{P}$ be a fixed finite set of predicates, each with a fixed arity.
A \emph{ground fact} is an atom of the form $p(c_1,\dots,c_k)$, where $p \in \mathcal{P}$ has arity $k \geq 0$ and each $c_i$ is a constant from a fixed finite domain of constants.  
Let $\mathcal{F}_p$ denote the set of all ground facts over predicate $p$, for all $p \in \mathcal{P}$, and $\mathcal{F}$ be the set of all ground facts, i.e., $\mathcal{F} = \bigcup_{p \in \mathcal{P}}{\mathcal{F}_p}$.
A \emph{database} $D$ is a finite set of ground facts, that is, $D \subseteq \mathcal{F}$.
Let $\mathcal{D}$ denote the domain of databases, defined as $2^{\mathcal{F}}$.
In the following, we assume databases grow monotonically during an extraction process, that is, along a finite sequence $D_0 \subseteq \dots \subseteq D_n$.

\paragraph{Queries.}
We consider \emph{boolean conjunctive queries} over databases.
A conjunctive query $q$ is a conjunction of (possibly negated) atoms whose variables are implicitly existentially quantified.
Given a database $D$, we say that $q$ is \emph{true} in $D$, written $D \models q$, if there exists a substitution $\sigma$ mapping the variables of $q$ to constants such that
(i) for every positive atom $a$ in $q$, $a\sigma \in D$, and
(ii) for every negated atom $\mathit{not}\ a$ in $q$, $a\sigma \notin D$.
A query $q$ is said to be \emph{monotone} if for all databases $D \subseteq D'$, $D \models q$ implies $D' \models q$.
In our setting, monotone queries are exactly those that do not contain negated atoms.
In later sections, monotone queries will be exploited to optimize logic-based control checks via caching.

\paragraph{Logic Programs.}
Let $\mathcal{LP}$ denote the set of logic programs over predicates in $\mathcal{P}$.
Specifically, we consider normal logic programs under the stable model semantics.
A \emph{normal rule} is an expression \emph{``head if body''} of the form
\[
a \leftarrow b_1,\dots,b_m,\ \mathit{not}\ c_1,\dots,\mathit{not}\ c_n,
\]
where $a$, $b_i$, and $c_j$ are atoms.
A \emph{normal program} $\Pi$ is a finite set of normal rules over predicates in $\mathcal{P}$.
Programs are interpreted under their standard ground instantiation with respect to the fixed finite domain of constants.
Given a ground program $\Pi$ and a set of ground atoms $I$, the
\emph{Gelfond--Lifschitz reduct} $\Pi^I$ is obtained by removing all rules whose body contains a negative literal $\mathit{not}\ c$ with $c \in I$, and by deleting all remaining negative literals from the bodies of the remaining rules.
A set of atoms $I$ is a \emph{stable model} (or \emph{answer set}) of $\Pi$ if $I$ is the least model of the reduct $\Pi^I$.
In the intended extraction setting, the programs are constructed so that $\Pi \cup D$ has a unique intended answer set.
We denote it by $\mathit{Ans}_{\Pi}(D)$.

\paragraph{Extraction Oracle.}
We abstract the behavior of a Large Language Model (LLM) by means of an \emph{extraction oracle}.
Let $\mathcal{P}_E = \{p_1,\dots,p_n\} \subseteq \mathcal{P}$ be a finite set of \emph{target predicates}, i.e., predicates whose facts are to be extracted from text.
When needed, we assume that the set of target predicates $\mathcal{P}_E$ is
equipped with a fixed ordering.
Formally, an extraction oracle is a function
\begin{align*}
    \begin{array}{rcl}
        \mathcal{O} : \mathcal{T} \times \mathcal{P}_E &{}\rightarrow{}& \mathcal{D}\\
        (T,p) &{}\mapsto{}& D
    \end{array}
\end{align*}
such that $D \subseteq \mathcal{F}_p$.
Intuitively, $\mathcal{O}(T,p)$ returns a (possibly empty) set of candidate facts for predicate $p$ extracted from the input text $T$.
The oracle may be incomplete or return spurious facts, and no assumptions are made regarding determinism, completeness, or monotonicity of its behavior.
Since at most one oracle call is performed per target predicate, the length of the extraction sequence is bounded by $|\mathcal{P}_E|$.

\paragraph{Prompt Transformations.}
We model prompt construction abstractly as text transformations.
Let $\Gamma : \mathcal{T} \to \mathcal{T}$ be a global transformation encoding macro-task or domain-level instructions, and let $\gamma_p : \mathcal{T} \to \mathcal{T}$ be predicate-specific transformation encoding instructions for extracting predicate $p \in \mathcal{P}_E$.
The extraction oracle is always applied to transformed texts of the form $\gamma_p(\Gamma(T))$, that is, when we write $\mathcal{O}(T,p)$ we actually mean $\mathcal{O}(\gamma_p(\Gamma(T)),p)$.
Stated differently, prompt transformations are modeled as text-to-text functions applied to the input text before oracle invocation;
for clarity of presentation, they are omitted from the notation.

\begin{example}[Running example]\label{ex:running:1}
Let $\Gamma$ be the macro-task function mapping a text $\underline{t}$ to 
{\begin{asp}
    Follow my instructions on how to map the INPUT to the given OUTPUT format. The INPUT text describes a (possibly layered) undirected graph. 
    INPUT: $\underline{t}$     OUTPUT:
\end{asp}}
Let $\gamma_{\mathtt{in\_layer/2}}$ be the transformation associated with \lstinline|in_layer/2|, mapping a text $\underline{t}$ to
{\begin{asp}
    $\underline{t}$ in_layer(layer, node).
    For each layer, list every node that belongs to that layer.
\end{asp}}
Let $T$ be the following unstructured text document:
{\begin{asp}
    The graph has three layers (0 to 2), each containing three nodes. Layer zero contains nodes n1, n2, and n3, while layer 2 contains nodes n7, and n9. There is an edge between n1 and n9, and node n4 is connected to n2 and n8.
\end{asp}}
So, if we call $\mathcal{O}(T,p)$, we are actually calling $\mathcal{O}(T',p)$, where $T' = \gamma_p(\Gamma(T))$ is
{\begin{asp}
    Follow my instructions $[\cdots]$ a (possibly layered) undirected graph.
    INPUT: The graph has $[\cdots]$ n2 and n8.     OUTPUT: in_layer(layer, node).
    For each layer, list every node that belongs to that layer.
\end{asp}}
The expected output are facts
\lstinline|in_layer(0,n1)|,
\lstinline|in_layer(0,n2)|,
\lstinline|in_layer(0,n3)|,
\lstinline|in_layer(2,n7)|, and
\lstinline|in_layer(2,n9)|,
while facts like 
\lstinline|in_layer(2,n8)|, and
\lstinline|in_layer(1,n5)| must not be produced in output as they are not explicitly mentioned in $T$.
{Specifically, we expect that the extraction process terminates reporting only facts that are actually present in the given input, avoiding any speculative interpolation of missing data.}
\hfill$\blacksquare$
\end{example}

\paragraph{Baseline Extraction Pipeline (LLMASP).}
As baseline, we consider a predicate-wise extraction pipeline, as exemplified by LLMASP-style approaches presented by \citeNP{DBLP_conf_ijcai_AlvianoGSR25}.
Given a fixed set of target predicates $\mathcal{P}_E$, we model the baseline extraction pipeline as a function
\begin{align*}
    \begin{array}{rcl}
        \mathsf{LLMASP} : \mathcal{T} \times \mathcal{D} \times \mathcal{LP} &{}\rightarrow{}& \mathcal{D} \\
        (T, D_0, \Pi) &{}\mapsto{}& \mathit{Ans}_\Pi(D_0 \cup \bigcup_{p \in \mathcal{P}_E}{\mathcal{O}(T,p)})
    \end{array}
\end{align*}
where $\mathsf{LLMASP}(T, D_0, \Pi)$ denotes the database produced by the baseline pipeline on input text $T$, program $\Pi$ and initial database $D_0$.
Operationally, the baseline invokes the extraction oracle independently for each target predicate $p \in \mathcal{P}_E$ on the (transformed) input text, collects all the extracted facts together with $D_0$ into a database
$D' = D_0 \cup \bigcup_{p \in \mathcal{P}_E}{\mathcal{O}(T,p)}$, 
and then applies ASP reasoning once using program $\Pi$, yielding the final database
$\mathit{Ans}_\Pi(D')$.
Note that, in this baseline setting, logical reasoning is applied only after the extraction phase and does not influence which target predicates are selected for extraction.
This formalization will allow us to precisely compare the baseline with the logic-guided extraction framework introduced next.

\begin{example}[Continuing Example~\ref{ex:running:1}]\label{ex:running:2}
Let $\mathcal{P}_E$ be $\{$\lstinline|in_layer/2|, \lstinline|layer_size/2|, \lstinline|edge/2|$\}$, 
$\mathcal{O}(T,$\lstinline|layer_size/2|$)$ be $\{$\lstinline|layer_size(0,3)|, \lstinline|layer_size(1,3)|, \lstinline|layer_size(2,3)|$\}$ and 
$\mathcal{O}(T,$\lstinline|edge/2|$)$ be $\{$\lstinline|edge(n1,n9)|, \lstinline|edge(n4,n2)|, \lstinline|edge(n4,n8)|$\}$.
Let $\Pi$ be the following program:
\begin{asp}
$r_1:$  edge(X,Y) :- edge(Y,X).
$r_2:$  warning(unmatched_layer_size, (Layer,Size,Cnt)) :-  layer_size(Layer,Size),
        Cnt = #count{Node : in_layer(Layer,Node)}, Size != Cnt.
\end{asp}
where $r_1$ computes the symmetric closure of \lstinline|edge/2| (the graph is undirected), and $r_2$ detects inconsistencies or missing data in the input text.
$\mathsf{LLMASP}(T,\emptyset,\Pi)$ returns a database including
\lstinline|warning(unmatched_layer_size, (1,3,0))| and
\lstinline|warning(unmatched_layer_size, (2,3,2))|.
In an interactive setting, the derived \lstinline|warning| facts may serve as explanations of inconsistencies or missing data, enabling the system to ask the user for clarification or further details.
\hfill$\blacksquare$
\end{example}

\section{Logic-Guided Extraction Framework}\label{sec:guided}

In the baseline pipeline introduced in Section~\ref{sec:background}, extraction and reasoning are strictly separated:
all target predicates are extracted independently by the oracle, and logical reasoning is applied only once as a post-processing step.
While this approach is simple and effective, it does not exploit logical knowledge during the extraction process itself.

In this section, we introduce a logic-guided extraction framework in which logical reasoning actively controls the extraction process.
The key idea is to interleave oracle calls with logic-based checks over the current database, allowing extraction decisions to depend on facts already obtained and on their logical consequences.
As a result, oracle calls for certain target predicates can be safely skipped without affecting the final extracted information.

The framework presented here generalizes the baseline pipeline and subsumes it
as a special case.
It provides a formal basis for reducing the number of oracle invocations while preserving the extracted facts under suitable assumptions, and it enables additional optimizations such as caching of logic-based checks.

\subsection{Intuition and Overview}\label{sec:intuition}

The logic-guided extraction framework builds on a simple observation:
in many domains, the extraction of certain predicates becomes unnecessary once sufficient information has already been obtained and logically processed.
In the baseline pipeline, however, such dependencies are ignored, as all target predicates are extracted independently before any reasoning takes place.
In contrast, our framework interleaves extraction and reasoning.
After each oracle invocation, the newly extracted facts are combined with the current database and processed by a logic program, yielding additional derived facts.
These derived facts can then be used to decide whether further oracle calls are needed.
Intuitively, if the truth of a target predicate can already be established (or ruled out) based on the current database and logical knowledge, invoking the oracle for that predicate becomes redundant.

To support such decisions, the framework relies on logic-based checks expressed as queries over the current database.
These queries act as \emph{guards} that determine whether a target predicate is eligible for extraction at a given stage.
Crucially, guards may depend both on facts explicitly extracted from the text and on facts derived by logical reasoning.
By construction, this approach may skip oracle calls that would be performed in the baseline pipeline.
Under suitable assumptions on the extraction oracle, however, skipping such calls does not affect the final set of extracted facts.
As a result, the logic-guided framework can reduce the number of oracle invocations while preserving the outcome of the baseline pipeline, and in practice may also mitigate the generation of spurious or hallucinated facts.

\subsection{Extraction Conditions and Guards}\label{sec:guards}

The logic-guided extraction framework is based on the notion of \emph{extraction conditions}, which determine whether invoking the extraction oracle for a target predicate is necessary at a given stage of the process.
These conditions are expressed declaratively as queries over the current database and its logical consequences.

\paragraph{Guards.}
Let $\mathcal{P}_E$ be the set of target predicates, equipped with a fixed ordering.
For each predicate $p \in \mathcal{P}_E$, let $\mathcal{G}_p$ be a finite set of conjunctive queries, called the \emph{guards} for $p$.
Given a database $D$, predicate $p$ is said to be \emph{admissible} for extraction in $D$ if all guards in $\mathcal{G}_p$ are satisfied in $D$, that is, if $D \models g$ for every $g \in \mathcal{G}_p$.
If at least one guard in $\mathcal{G}_p$ is not satisfied in $D$, the extraction of $p$ is skipped at that stage.
Guards are not required to be unique to a predicate and may appear in the guard sets of multiple target predicates.

\paragraph{Monotone and Non-monotone Guards.}
A guard $g \in \mathcal{G}_p$ is said to be \emph{monotone} if it does not contain negated atoms, as defined in Section~\ref{sec:background}.
Monotone guards depend only on the presence of facts in the database and, once satisfied, remain satisfied as the database grows.
Non-monotone guards may depend on the absence of facts and can change their truth value as new information becomes available.
This distinction is important for optimization purposes.
In particular, monotone guards admit persistent caching of positive outcomes, while non-monotone guards generally require re-evaluation when the database is extended.

\subsection{Logic-Guided Extraction Semantics}\label{sec:guided-semantics}

We now define the semantics of logic-guided extraction using the guard-based admissibility conditions introduced in Section~\ref{sec:guards}.
Extraction proceeds over a fixed ordering $\mathcal{P}_E = \langle p_1,\dots,p_n \rangle$ of target predicates.
Each predicate $p_i$ is associated with a logic program $\Pi_{p_i}$, which is applied after processing $p_i$.
We collectively refer to the ordering $\mathcal{P}_E$, the guards $\{\mathcal{G}_{p_i}\}_{i=1}^n$, and the programs $\{\Pi_{p_i}\}_{i=1}^n$ as a \emph{logic-guided extraction configuration}.
Let $\mathcal{C}$ denote the domain of all such configurations, and let $C \in \mathcal{C}$ be a fixed configuration.

\paragraph{Incremental Database Construction.}
The logic-guided extraction process constructs a sequence of databases
$D_0 \subseteq D_1 \subseteq \dots \subseteq D_n$,
where $D_0$ is the initial database.
For each $i \in \{1,\dots,n\}$, the database $D_i$ is obtained from $D_{i-1}$ by processing predicate $p_i$ according to its admissibility in $D_{i-1}$.
If predicate $p_i$ is admissible in $D_{i-1}$, then the extraction oracle is invoked for $p_i$, the extracted facts are added to the current database, and logical reasoning is applied using program $\Pi_{p_i}$.
If predicate $p_i$ is not admissible in $D_{i-1}$, the oracle invocation is
skipped and no logical reasoning is applied.
Formally,
\begin{align}\label{eq:lgx-step}
    D_i =
        \begin{cases}
            \mathit{Ans}_{\Pi_{p_i}}(D_{i-1} \cup \mathcal{O}(T,p_i)) & \text{if } D_{i-1} \models g \text{ for all } g \in \mathcal{G}_{p_i},\\
            D_{i-1} & \text{otherwise.}
        \end{cases}
\end{align}

\paragraph{Logic-Guided Extraction Function.}
The logic-guided extraction semantics defined above induces a function
\begin{align*}
    \begin{array}{rcl}
        \mathsf{LGX} : \mathcal{T} \times \mathcal{D} \times \mathcal{C} &{}\rightarrow{}& \mathcal{D} \\
        (T,D_0,C) &{}\mapsto{}& D_n
    \end{array}
\end{align*}
which maps an input text $T$, an initial database $D_0$, and a configuration
$C \in \mathcal{C}$ to the final database $D_n$ produced according to the
semantics defined above.

\begin{example}[Continuing Example~\ref{ex:running:2}]\label{ex:running:3}
Let $g$ be the guard \lstinline|layer_size(_,_)|, $g'$ be \lstinline|not warning(_,_)|, and
$C$ be the configuration 
\begin{align*}
\left(
    \tuple{
        \text{
            \lstinline|layer_size/2|,
            \lstinline|in_layer/2|,
            \lstinline|edge/2|
        }
    }, 
    \tuple{
        \emptyset,
        \{g\},
        \{g'\}
    },
    \tuple{
        \emptyset,
        \{r_2\},
        \{r_1\}
    }
\right)
\end{align*}
The target predicates are ordered so that \lstinline|in_layer/2| is extracted only if some layer information has been found, as enforced by the guard $g$.
Moreover, the program $\Pi$ is split so that warnings about unmatched layer sizes (rule $r_2$) are derived immediately after extracting \lstinline|in_layer/2|.
If such warnings arise, the guard $g'$ fails, thereby preventing the extraction of \lstinline|edge/2| at later stages.
\hfill$\blacksquare$
\end{example}

\paragraph{Baseline as a Special Case.}
In the proposed framework, a logic program can be seen as a degenerate logic-guided configuration in which no admissibility constraints are imposed and all logical reasoning is deferred to the final step.
From this perspective, the baseline extraction pipeline of Section~\ref{sec:background} is recovered as a special case of the above semantics.
Specifically, for an arbitrary ordering $\mathcal{P}_E = \langle p_1,\dots,p_n \rangle$, we associate an empty set of guards to every predicate, that is, $\mathcal{G}_{p_i} = \emptyset$ for all $i$.
As a consequence, every predicate is admissible for extraction and the oracle is invoked for each $p_i$.
Moreover, we associate the empty program with all predicates except the last one.
Formally, we set $\Pi_{p_i} = \emptyset$ for all $i < n$ and $\Pi_{p_n} = \Pi$, where $\Pi$ is the baseline logic program.
With this choice, no logical reasoning is applied during extraction.
All reasoning is deferred to the processing of the last predicate $p_n$.
In this baseline setting, the order of predicates in $\mathcal{P}_E$ does not affect the final result.
Accordingly, the input expected by $\mathsf{LGX}$ coincides exactly with the program parameter of the $\mathsf{LLMASP}$ function defined in Section~\ref{sec:background}.

\subsection{Baseline Equivalence}\label{sec:correctness}

We now establish that the logic-guided extraction framework is equivalent to the baseline extraction pipeline under a degenerate configuration and a natural assumption on the extraction oracle.

\paragraph{Oracle Assumption.}
We assume that the extraction oracle is \emph{guard-respecting}, in the following sense.
Let $p \in \mathcal{P}_E$ be a target predicate and let $D$ be a database.
If $D \not\models g$ for some guard $g \in \mathcal{G}_p$, then $\mathcal{O}(T,p) = \emptyset$.
Intuitively, when the admissibility conditions for a predicate are not satisfied, invoking the oracle for that predicate would not produce any extracted facts.

\begin{theorem}[Equivalence Theorem]\label{thm:equivalence}
Let $T \in \mathcal{T}$ be an input text and let $D_0 \in \mathcal{D}$ be an initial database.
Let $\Pi$ be a logic program, seen as a degenerate logic-guided configuration  as described in Section~\ref{sec:guided-semantics}.
Let $\mathsf{LGX}(T,D_0,\Pi)$ denote the result of logic-guided extraction under this configuration, and let $\mathsf{LLMASP}(T,D_0,\Pi)$ denote the result of the baseline extraction pipeline.
Under the oracle assumption above, it holds that
$\mathsf{LGX}(T,D_0,\Pi) = \mathsf{LLMASP}(T,D_0,\Pi)$.
\end{theorem}
\begin{proof}[Proof (sketch)]
Let $D_0 \subseteq \dots \subseteq D_n$ be the sequence of databases produced by
$\mathsf{LGX}(T,D_0,\Pi)$.
For all $i \in \{1,\ldots,n\}$, we have $\mathcal{G}_{p_i} = \emptyset$ by assumption.
Hence, $D_{i-1} \models g$ for all $g \in \mathcal{G}_{p_i}$ holds vacuously, and the update rule~\eqref{eq:lgx-step} reduces to
$D_i = \mathit{Ans}_{\Pi_{p_i}}(D_{i-1} \cup \mathcal{O}(T,p_i))$.

By construction of the degenerate configuration, $\Pi_{p_i} = \emptyset$ for all $i < n$, and $\Pi_{p_n} = \Pi$.
Moreover, $\mathit{Ans}_\emptyset(D) = D$ for any database $D$.
Therefore, for all $i \in \{1,\ldots,n\}$,
\[
    D_i =
        \begin{cases}
            D_{i-1} \cup \mathcal{O}(T,p_i) & \text{if } i < n,\\
            \mathit{Ans}_{\Pi}(D_{n-1} \cup \mathcal{O}(T,p_n)) & \text{if } i = n.
        \end{cases}
\]

Unfolding the construction of $D_n$, we obtain
\[
\mathsf{LGX}(T,D_0,\Pi) = D_n
= \mathit{Ans}_{\Pi}\!\left(D_0 \cup \bigcup_{i=1}^{n} \mathcal{O}(T,p_i)\right).
\]
By definition of $\mathsf{LLMASP}$, the right-hand side coincides with
$\mathsf{LLMASP}(T,D_0,\Pi)$.
\end{proof}

\subsection{Extraction Efficiency and Guard-Based Optimization}\label{sec:efficiency}

A key advantage of logic-guided extraction is that admissibility checks allow
the process to avoid unnecessary oracle invocations while preserving the final
outcome.
In $\mathsf{LLMASP}$, the oracle is invoked once for every target predicate,
whereas $\mathsf{LGX}$ invokes it only when guards are satisfied in the current
database.
Guard outcomes can be cached to avoid redundant guard evaluations (and thus
solver calls) without affecting correctness.

\begin{theorem}[Efficiency and Sound Guard Reuse]\label{thm:efficiency}
Let $D_0 \subseteq \dots \subseteq D_n$ be the database sequence constructed by
$\mathsf{LGX}(T,D_0,C)$.
Then $\mathsf{LGX}$ performs at most as many oracle calls as $\mathsf{LLMASP}$,
and strictly fewer whenever some target predicate is not admissible during the
extraction process.
Moreover, guard evaluations can be safely reused:
(i) if $D_i = D_j$ with $i<j$, then $D_j \models g$ if and only if $D_i \models g$;
and (ii) if $g$ is monotone and $D_i \models g$ for some $i \in \{0,\dots,n-1\}$,
then $D_j \models g$ for all $j \geq i$.
\end{theorem}
\begin{proof}[Proof (sketch)]
Oracle call reduction is immediate: $\mathsf{LLMASP}$ invokes the oracle once for
every target predicate in $\mathcal{P}_E$, whereas $\mathsf{LGX}$ invokes it only
when the corresponding guards are satisfied.
For guard reuse, (i) follows from the fact that conjunctive queries have the
same truth value on identical databases, and (ii) is a consequence of the monotonicity of $g$ and because $D_i \subseteq D_j$ by construction.
\end{proof}

The above theorem justifies execution strategies that cache guard outcomes
according to these two principles.
Such caching always agrees with fresh guard evaluation, and therefore preserves
admissibility decisions and yields the same final database as the uncached
semantics of $\mathsf{LGX}$.
Consequently, caching affects only efficiency, not correctness.

\section{Implementation}\label{sec:implementation}

This section presents a concrete realization of the logic-guided extraction framework defined in Section~\ref{sec:guided}, given in Algorithm~\ref{alg:lgx}, and shows how its declarative semantics can be operationalized in practice.
Implementation choices such as guard caching affect only efficiency and do not alter the formal semantics of the framework.

\begin{algorithm}[t]
\DontPrintSemicolon
\SetKwInOut{Input}{Input}
\SetKwInOut{Output}{Output}
\caption{\textsc{LGX-Exec}$(T,D_0,C)$: Logic-guided extraction with guard caching}\label{alg:lgx}
\Input{Text $T$, initial database $D_0$,
configuration $C = (\langle p_1,\dots,p_n\rangle,\{\mathcal{G}_{p_i}\},\{\Pi_{p_i}\})$}
\Output{Final database $D$}

$(D, \mathcal{K}^{\top}_{m}, \mathcal{K}^{\bot}_{m}, \mathcal{K}^{\top}_{n}, \mathcal{K}^{\bot}_{n}) := (D_0, \emptyset, \emptyset, \emptyset, \emptyset)$\tcp*[f]{result and cache initialization}

\For(\tcp*[f]{process target predicates}){$i := 1$ \KwTo $n$}{
  $\mathit{admissible} := \top$\;
  
  \ForEach{guard $g \in \mathcal{G}_{p_i}$}{
      \uIf(\tcp*[f]{cache hit}){$g \in \mathcal{K}^{\top}_{\mathit{mono}(g)} \cup \mathcal{K}^{\bot}_{\mathit{mono}(g)}$}{
        $\mathit{hold} := g \in \mathcal{K}^{\top}_{\mathit{mono}(g)}$\tcp*[f]{retrieve answer from cache}
      }
      \Else(\tcp*[f]{cache miss}){
        $\mathit{hold} := \mathit{q} \in \mathit{Ans}_{\{\mathit{q} \leftarrow g\}}(D)$\tcp*[f]{answer query}
        
        $\mathcal{K}^{\mathit{hold}}_{\mathit{mono}(g)} := \mathcal{K}^{\mathit{hold}}_{\mathit{mono}(g)} \cup \{g\}$\tcp*[f]{cache answer}
      }

    \If(\tcp*[f]{if some guard is false}){$\mathit{hold} = \bot$}{
      $\mathit{admissible} := \bot$\tcp*[f]{$p_i$ is not admissible}
    }

  }
  
  \If{$\mathit{admissible}$}{
    $D' := \mathit{Ans}_{\Pi_{p_i}}(D \cup \mathcal{O}(T, p_i))$\tcp*[f]{extract and derive facts}
    
    \If(\tcp*[f]{if there are new facts}){$D' \neq D$}{
      $(D, \mathcal{K}^{\bot}_{m}, \mathcal{K}^{\top}_{n}, \mathcal{K}^{\bot}_{n}) := (D', \emptyset, \emptyset, \emptyset)$\tcp*[f]{update DB, invalidate cache}
    }
  }
}

\Return{$D$}\;

\end{algorithm}

\paragraph{Guard evaluation.}
For a predicate $p_i$, admissibility is determined by evaluating the guards in $\mathcal{G}_{p_i}$ against the current database $D$.
A guard $g$ is evaluated by checking whether a fresh atom $q$ is derivable from the program $\{q \leftarrow g\}$ together with the facts in $D$, that is, whether $q \in \mathit{Ans}_{\{ q \leftarrow g \}}(D)$.
Guard evaluation conceptually corresponds to the conjunction of the guards in $\mathcal{G}_{p_i}$.
As soon as one guard is found to be false, the predicate $p_i$ is deemed not admissible and no further guard evaluations are required.
If all guards hold, $p_i$ is admissible, the extraction oracle is invoked, and the facts returned by the oracle are combined with the current database and processed by the associated logic program $\Pi_{p_i}$.

\paragraph{Guard caches.}
In what follows, we use $\top$ and $\bot$ to denote the Boolean outcomes
\emph{true} and \emph{false}.
To optimize guard evaluation, the algorithm maintains a family of cache sets $\mathcal{K}^{\tau}_{\mu}$, where $\tau \in \{\top,\bot\}$ denotes the truth value of a guard and $\mu \in \{m,n\}$ denotes whether the guard is monotone or non-monotone.
A guard $g$ belongs to $\mathcal{K}^{\tau}_{\mu}$ if it has previously been evaluated with truth value $\tau$ and $\mathit{mono}(g)=\mu$, where $\mathit{mono}$ determines the monotonicity of guards.
Cached results are reused whenever possible.
When the database is extended, cached results for non-monotone guards and cached negative results for monotone guards are invalidated, while cached positive results for monotone guards are preserved.
This policy is sound by Theorem~\ref{thm:efficiency} and affects only the number of guard evaluations performed, without altering admissibility decisions or the final database produced.

\section{Experimental Evaluation}
\label{sec:experiments}

In this section, we empirically evaluate the $\mathsf{LGX}$ framework to assess its effectiveness in reducing LLM oracle calls, its impact on extraction quality, and the efficiency of the proposed caching strategies.
{Source files and datasets are available online at \url{https://github.com/Xiro28/LGX}.}

\paragraph{Implementation and Setup.}
From an implementation standpoint, we observe how the two cache sets $\mathcal{K}^\bot_m$ and $\mathcal{K}^\bot_n$ can be safely combined, since entries of both sets require re-evaluation once the database has changed. 
Our framework is implemented using \textsc{clingo~5.8} by \citeNP{DBLP:journals/tplp/GebserKS11} as the underlying ASP solver, and \textsc{Meta Llama~3.1} proposed by \citeNP{DBLP:journals/corr/abs-2407-21783} as the extraction oracle, deployed in two configurations: 
a lightweight version (\texttt{8B}, hereafter referred to as \emph{Small}) and a medium-capacity version (\texttt{70B}, referred to as \emph{Medium}).
{Additionally, we performed robustness tests with a third high-capacity configuration based on \textsc{GPT-OSS} presented by \citeNP{openai2025gptoss120bgptoss20bmodel}, denoting its \texttt{120B} version as \emph{Large}, to evaluate the consistency of the framework’s performance across model architectures.}

To ensure a fair and rigorous comparison, we simulated the $\mathsf{LLMASP}$ baseline using the $\mathsf{LGX}$ engine itself, adopting the degenerate configuration described in Section~3.4. 
This approach eliminates potential implementation biases, ensuring that performance deltas are solely attributable to the logic-guided strategy. 
Furthermore, we minimized stochastic variance by setting the LLM sampling temperature to $0$ and employing a persistent cache for oracle calls. 
This guarantees that identical queries yield the same candidate facts across different system runs, effectively isolating the impact of the logic-guided control flow on both extraction efficiency and quality.

\begin{table}[t]
\centering
\caption{
Experimental results on Graph (G) and Logic Puzzle (LNRS) benchmarks. We report the number of Oracle interactions ($\mathcal{O}$ Calls), the enforcement strategy, and accuracy metrics (F1-Score, Perfect Rate).
}
\label{tab:main_results}
\setlength{\tabcolsep}{3.5pt}
\begin{tabular}{cclclcc}
\toprule
\textbf{Bench.} & \textbf{$\mathcal{O}$} & \textbf{Pipeline} & \textbf{$\mathcal{O}$ Calls} & \textbf{Guard Enforcement} & \textbf{F1-Score} & \textbf{Perfect Rate} \\ 
\midrule
\multirow{11}{*}{{G}} 
  & \multirow{3}{*}{\shortstack{Small\\(\texttt{8B})}} 
      & \multirow{2}{*}{$\mathsf{LLMASP}$} & \multirow{2}{*}{240} & None & 0.800 & 0.317 \\
  & & & & \emph{A posteriori} & \textbf{0.893} & \textbf{0.617} \\
  \addlinespace[0.4em]
  & & $\mathsf{LGX}$ & \textbf{203} & By design & \textbf{0.893} & \textbf{0.617} \\
  \cmidrule(lr){2-7}

  & \multirow{3}{*}{\shortstack{Medium\\(\texttt{70B})}} 
      & \multirow{2}{*}{$\mathsf{LLMASP}$} & \multirow{2}{*}{240} & None & 0.870 & 0.417 \\
  & & & & \emph{A posteriori} & \textbf{0.952} & \textbf{0.850} \\
  \addlinespace[0.4em]
  & & $\mathsf{LGX}$ & \textbf{193} & By design & \textbf{0.952} & \textbf{0.850} \\ 

  \cmidrule(lr){2-7}

  & \multirow{3}{*}{\shortstack{Large\\(\texttt{120B})}} 
      & \multirow{2}{*}{$\mathsf{LLMASP}$} & \multirow{2}{*}{240} & None & 0.944 & 0.667 \\
  & & & & \emph{A posteriori} & \textbf{1.000} & \textbf{1.000} \\
  \addlinespace[0.4em]
  & & $\mathsf{LGX}$ & \textbf{186} & By design & \textbf{1.000} & \textbf{1.000} \\

\midrule

\multirow{7}{*}{{LNRS}} 
  & \multirow{3}{*}{\shortstack{Small\\(\texttt{8B})}} 
      & \multirow{2}{*}{$\mathsf{LLMASP}$} & \multirow{2}{*}{2998} & None & 0.440 & 0.000 \\
  & & & & \emph{A posteriori} & \textbf{0.842} & \textbf{0.023} \\
  \addlinespace[0.4em]
  & & $\mathsf{LGX}$ & \textbf{876} & By design & \textbf{0.842} & \textbf{0.023} \\
  \cmidrule(lr){2-7}

  & \multirow{3}{*}{\shortstack{Medium\\(\texttt{70B})}} 
      & \multirow{2}{*}{$\mathsf{LLMASP}$} & \multirow{2}{*}{3072} & None & 0.572 & 0.000 \\
  & & & & \emph{A posteriori} & \textbf{0.963} & \textbf{0.469} \\
  \addlinespace[0.4em]
  & & $\mathsf{LGX}$ & \textbf{864} & By design & \textbf{0.963} & \textbf{0.469} \\ 

  \cmidrule(lr){2-7}

  & \multirow{3}{*}{\shortstack{Large\\(\texttt{120B})}} 
      & \multirow{2}{*}{$\mathsf{LLMASP}$} & \multirow{2}{*}{3072} & None & 0.632 & 0.000 \\
  & & & & \emph{A posteriori} & \textbf{0.947} & \textbf{0.555} \\
  \addlinespace[0.4em]
  & & $\mathsf{LGX}$ & \textbf{864} & By design & \textbf{0.947} & \textbf{0.555} \\ 

\bottomrule
\end{tabular}
\end{table}

\paragraph{Benchmarks.}
Our evaluation relies on two primary benchmarks designed to assess different aspects of the logic-guided extraction.

\noindent \textbf{Graph Benchmark (G).}
Based on our running example of layered graphs, this benchmark compares \textsc{LGX} against the \textsc{LLMASP} baseline under two \emph{guard enforcement} strategies (see Table~\ref{tab:main_results}):
(i) \emph{None}, representing a na\"ive, blind extraction; and
(ii) \emph{A posteriori}, where invalid results are pruned post-hoc using manually injected constraints.

\noindent \textbf{Logic Puzzle Benchmark (LNRS).}
Adapted from the work by \citeNP{DBLP_conf_ijcai_AlvianoGSR25}, this dataset aggregates four ASP domains (Labyrinth, Nomystery, Ricochet Robots and Sokoban) to evaluate mixed-domain handling.
In this setting, an initial \emph{domain-identifier} extraction acts as a hierarchical guard, restricting all subsequent oracle calls to the specific puzzle schema and preventing cross-domain noise.

\paragraph{Performance and Efficiency Analysis.}
Table~\ref{tab:main_results} summarizes the experimental results on the Graph (G) and Logic Puzzle (LNRS) benchmarks. 
The empirical data supports three key observations regarding the impact of $\mathsf{LGX}$:
\begin{itemize}
    \item 
    \textbf{Drastic Reduction in Traffic.} 
    $\mathsf{LGX}$ significantly lowers the number of interactions with the oracle compared to the baseline. 
    In the LNRS benchmark, calls are reduced by approximately 72\% (dropping from 3072 to 864 calls with the medium oracle). 
    This confirms that the logic-guided pruning effectively identifies and discards inconsistent branches before expensive queries are made.

    \item \textbf{Accuracy Equivalence.} 
    Consistent with Theorem~\ref{thm:equivalence}, the accuracy of $\mathsf{LGX}$ (enforcement by design) is identical to the \emph{a posteriori} baseline. 
    This empirically verifies that our pruning strategy is sound: it optimizes efficiency without compromising the quality of the solution.

    \item \textbf{Scalability with Complexity.} 
    The efficiency gap widens significantly as the logical structure deepens. 
    While the reduction is moderate for the structurally simpler G benchmark ($\sim$15--20\%), it becomes substantial for the hierarchically structured LNRS benchmark. 
    This confirms that $\mathsf{LGX}$ excels in rich domains, where guard dependencies enable aggressive pruning of oracle calls.
\end{itemize}

\noindent \textbf{Quality Analysis (Benchmark G).}
{Na\"ive extraction generates spurious facts, such as invalid layers not explicitly mentioned in the input text, or facts that are not expected to be extracted due to failing guards.}
These structural inconsistencies severely impact F1-score and Perfect Rate.
While $\mathsf{LLMASP}$ can recover quality via \emph{a posteriori} filtering, $\mathsf{LGX}$ ensures this consistency \emph{by design}.
This allows $\mathsf{LGX}$ to achieve optimal quality with strictly fewer oracle calls (Table~\ref{tab:main_results}), unlike the fixed-cost baseline.

\noindent \textbf{Efficiency Analysis (Benchmark LNRS).}
Efficiency is critical in the mixed-domain LNRS benchmark.
While the baseline indiscriminately queries all predicates, incurring high costs, $\mathsf{LGX}$ matches the optimal accuracy of the \emph{a posteriori} configuration at a fraction of the expense.
By proactively pruning irrelevant queries rather than filtering results post-hoc, $\mathsf{LGX}$ avoids the prohibitive overhead of the standard pipeline.

\noindent \textbf{ASP Solver Overhead.}
Our caching strategy significantly reduces internal reasoning costs.
In LNRS, shared guards allow $\mathsf{LGX}$ to avoid between 61.4\% and 82.6\% of redundant ASP solver calls.
Crucially, even where caching is less effective (Benchmark G), the latency introduced by guard checks is negligible compared to the massive runtime savings achieved by reducing oracle interactions.

\section{Related Work}\label{sec:rw}

Existing work on combining LLMs with logic-based reasoning can be categorized according to the role assigned to the symbolic component. 
In some approaches, logic is the target representation generated from text; in others, it is used as a validation layer, or as part of a probabilistic neuro-symbolic semantics.
We discuss these lines of work with emphasis on their relation to extraction from natural language and to the use of ASP as a reasoning component.

\paragraph{Logic as a target language.}
A first class of approaches uses natural-language processing, and more recently LLMs, to construct a symbolic representation that is then given to a
logic-based solver. 
This view extends earlier work in which natural-language content was encoded directly in ASP or related declarative languages. 
For example, \citeNP{DBLP:journals/tplp/PendharkarBSG22} represent English passages through a Neo-Davidsonian formalism and use common-sense knowledge from WordNet, a lexical database presented by \citeNP{DBLP:journals/cacm/Miller95}, to answer questions over the resulting program.
With the advent of LLMs, part of the symbolic encoding can be delegated to neural generation. 
\citeNP{DBLP:conf/kr/IshayY023} introduce a Generate--Define--Test methodology to obtain entities, relations, and rules from text; \citeNP{DBLP:conf/aaai/Ishay025} extend this direction to action languages for temporal and causal reasoning.
Similarly, the STAR framework of \citeNP{DBLP:conf/iclp/RajasekharanZ0G23} translates natural-language knowledge into a representation processed by s(CASP), supporting goal-directed and explainable derivations. 
These approaches differ from ours because the LLM is primarily used to construct the symbolic theory itself.
In our setting, the symbolic schema and the control knowledge are fixed by the framework, while the LLM is used to produce candidate facts.

\paragraph{Logic-based extraction with fixed schemas.}
A second line of work assumes a fixed relational or logical schema and uses the LLM to populate it from natural language. The [LLM]+ASP framework of \citeNP{DBLP:conf/acl/YangI023} and LLMASP by \citeNP{DBLP_conf_ijcai_AlvianoGSR25} are representative of this direction and motivate the predicate-wise abstraction adopted earlier in the paper.
Related ideas also appear in systems such as AutoCompanion by \citeNP{DBLP:journals/tplp/ZengRBWAG24}, where s(CASP) is used as a symbolic component in a dialogue setting. 
These approaches provide an important reference point for our work because they separate extraction from reasoning through a fixed schema. The distinction is that, in our framework, part of the symbolic reasoning is moved from post-processing to extraction control: logical conditions are evaluated during the construction of the database and can affect which extraction requests are issued.

\paragraph{Probabilistic neuro-symbolic reasoning.}
A distinct family of neuro-symbolic systems integrates neural predictions directly into a probabilistic logical semantics. DeepProbLog by \citeNP{DBLP:conf/nips/ManhaeveDKDR18} extends probabilistic logic programming with neural predicates, and NeurASP by \citeNP{DBLP:conf/ijcai/YangIL20} combines ASP with probability distributions associated with atoms generated by neural classifiers. These frameworks provide a principled account of uncertain neural outputs and support forms of learning or probabilistic inference. 
Our setting is different in two respects. First, the LLM is treated as a pre-trained black-box oracle rather than as a trainable neural component integrated into the semantics. 
Second, extracted atoms are handled as crisp candidate facts, and the role of ASP is to derive consequences, check consistency, and control admissibility rather than to assign or combine probabilities.

\paragraph{Post-hoc validation and refinement.}
Several approaches use symbolic reasoning, execution feedback, or self-reflective procedures only after a neural model has generated an output.
For instance, \citeNP{DBLP:conf/ijcai/EiterGHO23} use ASP to compute  contrastive explanations for visual question-answering results. 
In the LLM literature, Self-Refine by \citeNP{DBLP:conf/nips/MadaanTGHGW0DPY23} and Reflexion by \citeNP{DBLP:conf/nips/ShinnCGNY23} introduce iterative feedback loops in which generated answers are critiqued and revised. 
These methods can improve output quality by detecting or correcting problematic outputs after generation. 
The framework proposed here addresses a complementary problem: it uses logic before selected extraction calls are made, so that some irrelevant or inadmissible queries are never issued to the oracle.


\section{Conclusion}\label{sec:conclusion}


We introduced a logic-guided data extraction framework that interleaves LLM-based extraction with Answer Set Programming, using logic-based admissibility conditions to actively control the extraction process.
By exploiting logical dependencies during extraction, the proposed framework reduces the number of oracle invocations while preserving the final extracted information under mild assumptions.
We formally characterized the framework, established its equivalence with a baseline extraction pipeline, and identified efficiency properties that enable practical optimizations.
We presented an execution algorithm that realizes the proposed semantics and exploits guard caching based on monotonicity and {unchanged database states}.
An experimental evaluation on ASP-derived benchmarks confirms that the framework substantially reduces the number of LLM calls and, in practice, mitigates the generation of spurious extracted facts.

Several directions for future work remain.
On the implementation side, guard monotonicity is currently treated as an external property provided by the configuration.
An interesting extension would be to automatically infer monotonicity properties of guards, enabling more aggressive and transparent caching strategies without additional user annotation.
More broadly, we plan to investigate dynamic guard evaluation strategies, incremental reasoning techniques, and tighter integrations between logic-based control and prompt construction.
In particular, logic-driven templating approaches for prompt generation (such as the Mustache-based mechanisms adopted in ASP Chef by \citeNP{DBLP:journals/tplp/AlvianoRF25}) could be leveraged to dynamically adapt extraction prompts based on derived facts and admissibility conditions.
These directions further strengthen the role of non-monotonic reasoning as a principled control mechanism for LLM-based data extraction.
{Furthermore, exploring the integration of goal-directed engines (e.g., \textsc{Prolog} systems with incremental tabling) for guard evaluation presents a promising path toward more granular and reactive extraction strategies, effectively combining local dependency tracking with ASP global constraint-solving capabilities.}

{
\section*{Acknowledgments}
This work was supported 
by the Italian Ministry of Health (MSAL)
    under POS projects \linebreak CAL.HUB.RIA (CUP H53C22000800006) and RADIOAMICA (CUP H53C22000650006);
by the Italian Ministry of Enterprises and Made in Italy
    under project STROKE 5.0 (CUP \linebreak B29J23000430005);
    under PN RIC project ASVIN ``Assistente Virtuale Intelligente di Negozio'' (CUP B29J24000200005);
by Regione Calabria
    NextGenGuides ``Innovazione Tecnologica per le Guide Turistiche del Futuro tramite l'ausilio di tecnologie innovative AI e sistemi di geolocalizzazione avanzata'' (CUP J39I24002040005);
    and MOZART ``Modelli e tecniche di mOdernizzazione di applicaZioni e data silos a supporto di clienti esterni, field force e impiegati di backoffice basati su AI GeneRativa e apprendimento auTomatico'' (CUP J49I24001740005);
and by the LAIA lab (part of the SILA labs). 
Mario Alviano is member of Gruppo Nazionale Calcolo Scientifico-Istituto Nazionale di Alta Matematica (GNCS-INdAM).
}

\bibliographystyle{acmtrans}
\bibliography{tclp_manuscript}

\label{lastpage}
\end{document}